\documentclass[letterpaper, 10 pt, conference]{ieeeconf}
\usepackage{amsmath,amsfonts}
\usepackage{array}
\usepackage[caption=false,font=normalsize,labelfont=sf,textfont=sf]{subfig}
\usepackage{textcomp}
\usepackage{stfloats}
\usepackage{url}
\usepackage{verbatim}
\usepackage{graphicx}
\def\BibTeX{{\rm B\kern-.05em{\sc i\kern-.025em b}\kern-.08em
    T\kern-.1667em\lower.7ex\hbox{E}\kern-.125emX}}
\usepackage{balance}
\usepackage{bm}
\usepackage{cite}

\usepackage{booktabs}
\usepackage{multirow}
\usepackage{makecell}

\usepackage[T1]{fontenc}

\newtheorem{thm}{Theorem}

\newtheorem{remark}{Remark}
\newtheorem{definition}{Definition}

\usepackage{varwidth}
\usepackage{algorithm,algpseudocode}

\IEEEoverridecommandlockouts                              
\title{\LARGE \bf
Dynamic Log-Gaussian Process Control Barrier Function for Safe Robotic Navigation in Dynamic Environments
}

\author{Xin Yin, Chenyang Liang, Yanning Guo and Jie Mei, \IEEEmembership{Member, IEEE}
\thanks{Xin Yin and Chenyang Liang contributed equally to this work.}%
\thanks{Xin Yin, Chenyang Liang and Jie Mei are with the School of Intelligence Science and Engineering, Harbin Institute of Technology, Shenzhen, Guangdong 518055, China.}%
\thanks{Yanning Guo is with the Department of Control Science and Engineering, Harbin Institute of Technology, Heilongjiang 150001, China.}%
\thanks{This work was supported in part by the Shenzhen Fundamental Research Program under Grant JCYJ20241202124010014, in part by the Guangdong Basic and Applied Basic Research Foundation under Grant 2023B1515120018 and Grant 2024B1515040008, and in part by the National Natural Science Foundation of China under Grant 62525304 and Grant U23B2036. (Corresponding author: Jie Mei, jmei@hit.edu.cn.)}
}

\begin{document}

\maketitle
\thispagestyle{empty}
\pagestyle{empty}

\newlength\mylen
\newcommand\myinput[1]{%
  \settowidth\mylen{\KwIn{}}%
  \setlength\hangindent{\mylen}%
  \hspace*{\mylen}#1\\}

  \let\oldnl\nl
  \newcommand{\nonl}{\renewcommand{\nl}{\let\nl\oldnl}}

\begin{abstract}
  Control Barrier Functions (CBFs) have emerged as efficient tools to address the safe navigation problem for robot applications.
  However, synthesizing informative and obstacle motion-aware CBFs online using real-time sensor data remains challenging, particularly in unknown and dynamic scenarios. 
  Motived by this challenge, this paper aims to propose a novel Gaussian Process-based formulation of CBF, termed the Dynamic Log Gaussian Process Control Barrier Function (DLGP-CBF),
  to enable real-time construction of CBF which are both spatially informative and responsive to obstacle motion.  
  Firstly, the DLGP-CBF leverages a logarithmic transformation of GP regression to generate smooth and informative barrier values and gradients, even in sparse-data regions. 
  Secondly, by explicitly modeling the DLGP-CBF as a function of obstacle positions, the derived safety constraint integrates predicted obstacle velocities, allowing the controller to proactively respond to dynamic obstacles' motion.  
  Simulation results demonstrate significant improvements in obstacle avoidance performance, including increased safety margins, smoother trajectories, and enhanced responsiveness compared to baseline methods.  
\end{abstract}

\section{Introduction}

Ensuring obstacle avoidance in unknown and dynamic environments containing both static and dynamic obstacles with unknown positions and velocities is essential for real-world robotic applications.  
Recently, Control Barrier Functions (CBFs) have emerged as an effective and powerful framework for enforcing safety constraints in robotic navigation tasks due to their simplicity and computational efficiency \cite{ames2019control}.  
A CBF defines a safe set as the super-zero-level set of a scalar function, whose non-negativity guarantees that the system state remains within the safe region.  
By integrating the safety constraint derived from the CBF into optimization-based controllers, such as Quadratic Programming (QP) \cite{wang2017safety}, Model Predictive Control (MPC) \cite{jian2023dynamic}, 
robotic systems can achieve collision-free navigation while pursuing task objectives.
While the CBF-based safety constraint can be directly incorporated into these optimization problems once a valid CBF is available,  
the primary challenge lies in synthesizing CBFs that are both informative and reliable.

Typically, synthesizing CBFs requires accurate geometric information of obstacles, such as their explicit locations and shapes \cite{wang2017safety,autenrieb2023safe,lie2023formation}.
However, such prior knowledge is usually unavailable in unknown scenarios. 
Consequently, real-time online synthesis of CBFs directly from onboard sensor data has become a necessity in practical tasks.

Existing online sensor-based CBF synthesis methods can be broadly categorized into two main classes: explicit geometric-based methods and implicit data-driven methods. 
Explicit geometric-based methods reconstruct obstacle shapes from sensor data using simplified geometric primitives. 
For instance, studies \cite{jian2023dynamic,zhang2025robust} utilize Minimum Bounding Ellipses (MBEs) to approximate LiDAR-observed obstacles, subsequently constructing multiple CBF constraints for each obstacle.
However, such geometric approximations typically yield overly conservative safe sets, as obstacles are enclosed by coarse bounding shapes.
Moreover, multiple safety constraints lead to increased computational complexity and poor scalability in environments with dense obstacles \cite{zhang2024online}.

In contrast, implicit data-driven methods \cite{dawson2022learning,srinivasan2020synthesis,9981177,9888130,keyumarsi2023lidar,10160805} synthesize CBFs by implicitly learning obstacle geometry from sensor data through supervised machine learning techniques.
These methods usually model the environment with a single unified barrier function, thus significantly reducing the computational burden associated with multiple constraints. 
Authors in \cite{srinivasan2020synthesis} employ a support vector machine (SVM) classifier to categorize sampled states as safe or unsafe, from which the CBF is synthesized. 
Similarly, the authors in \cite{9981177} utilize a deep neural network to directly update the CBF from instantaneous LiDAR measurements.
However, these approaches \cite{srinivasan2020synthesis,9981177} typically demand extensive labeled data and significant computational resources for online training, restricting their practical real-time applicability.

Recently, Gaussian Process (GP)-based methods \cite{9888130,keyumarsi2023lidar} have demonstrated significant potential for online CBF synthesis, owing to their ability to perform regression without requiring large-scale labeled datasets, as well as their low training cost and inference efficiency.  
Unlike other parameterized learning methods such as SVM \cite{srinivasan2020synthesis} and neural network \cite{9981177}, GP regression is a non-parametric model that relies on the correlation between training points and query inputs to produce predictions \cite{lyu}, making it particularly suitable for scenarios with sparse and noisy sensor data.
In \cite{keyumarsi2023lidar}, the constant-mean GP regression is employed to synthesize the CBF as a function of the robot's position.
The training dataset is constructed using the positions of the obstacle boundary points extracted from instantaneous LiDAR data.
However, existing GP-based approaches \cite{9888130,keyumarsi2023lidar} exhibit two critical limitations when applied to safe robot navigation tasks.
First, these methods are inherently reactive, as they treat the obstacle configuration as static within each control cycle.
The GP-based CBF is updated solely on instantaneous LiDAR observations without accounting for obstacle motion or future positions. 
This limitation makes the resulting CBFs less effective for anticipating potential collisions, especially when obstacles are moving rapidly or unpredictably.
Second, the standard constant-mean GP regression adopted in these approaches \cite{9888130,keyumarsi2023lidar} tends to return to its prior mean value in regions far from observed data.
This behavior leads to the saturation of predicted CBF values and vanishing gradients when the robot is not in close proximity to any obstacle. 
As a result, the CBF tends to provide insufficient safety information in such regions, potentially causing delayed avoidance responses.

Motivated by the above discussions, this paper proposes a novel GP-based CBF formulation, termed the \emph{Dynamic Log Gaussian Process Control Barrier Function (DLGP-CBF)}, to enhance safe robot navigation in unknown and dynamic environments.
The main contributions are summarized as follows.
\begin{itemize}
  \item We introduce a logarithmic GP-based CBF formulation that ensures informative and non-vanishing barrier values and gradients, even in regions distant from observed obstacles.
  \item We explicitly incorporate predicted obstacle velocities into the CBF formulation by modeling the CBF as a function of obstacle positions, enabling proactive and motion-aware collision avoidance.
  \item We validate the proposed method through simulation experiments in dynamic environments, demonstrating superior safety margins, responsiveness, and trajectory smoothness compared to existing baselines.
\end{itemize}

The remainder of this paper is organized as follows. 
Section \ref{sec:preliminaries} introduces preliminaries and the problem description. 
Section \ref{sec:methodology} details the proposed DLGP-CBF synthesis and the associated control framework. 
Simulation results validating the effectiveness of the proposed approach are presented in Section \ref{sec:simulation}.

\section{Preliminaries and Problem Description} \label{sec:preliminaries}

\subsection{Preliminaries}
\subsubsection{Dynamic Control Barrier Function}
Consider a general control-affine system with dynamics
\begin{equation}\label{eq:control_affine}
  \begin{aligned}
    \dot{x} &= f(x) + g(x)u,
  \end{aligned}
\end{equation}
where $x \in \mathcal{X} \subset \mathbb{R}^n$ is the state, $u \in \mathbb{R}^m$ is the control input, $f: \mathbb{R}^n \to \mathbb{R}^n$ and $g: \mathbb{R}^n \to \mathbb{R}^{n \times m}$ are locally Lipschitz continuous functions.
For safety-critical control, we define a dynamic safe set $\mathcal{C}(t)$ as the zero super-level set of a continuously differentiable function $h: \mathcal{X} \times \mathbb{R}^+ \to \mathbb{R}$, i.e.,
\begin{equation*}
  \begin{aligned}
    \mathcal{C}(t) &= \{x \in \mathcal{X} \mid h(x,t) \geq 0\}.
  \end{aligned}
\end{equation*}
The system \eqref{eq:control_affine} is defined to be \textit{safe} if the state $x(t)$ remains within the safe set $\mathcal{C}(t)$ for all time $t \geq 0$.
We introduce the dynamic control barrier function (D-CBF) $h(x,t)$ as follows.
\begin{definition}[\cite{jian2023dynamic}]
A continuously differentiable function $h: \mathcal{X} \times \mathbb{R}^+ \to \mathbb{R}$ is a D-CBF for the system \eqref{eq:control_affine} if there exists an extended class $\mathcal{K}_\infty $ function $\alpha(\cdot)$ such that the following inequality holds for all $(x,t) \in \mathcal{X} \times \mathbb{R}^+$:
  \begin{equation}\label{eq:dcf}
    \begin{aligned}
      \sup_{u \in \mathbb{R}^m} [ \mathcal{L}_fh(x,t) +\mathcal{L}_g h(x,t)u  + \frac{\partial h(x,t)}{\partial t} + \alpha(h(x,t))] \geq 0,
        \end{aligned}
  \end{equation}
  where $\mathcal{L}_fh$ and $\mathcal{L}_gh$ are the Lie derivatives of $h(x,t)$ with respect to $f(x)$ and $g(x)$, respectively.
\end{definition}

Let ${K}_{cbf}(x,t)$ denote the set of control inputs that satisfy the D-CBF condition \eqref{eq:dcf}, i.e.,
\begin{equation*}
  \begin{aligned}
    {K}_{cbf}(x,t) = \{u \in \mathbb{R}^m \mid \mathcal{L}_fh(x,t) +\mathcal{L}_g h(x,t)u   \\
    + \frac{\partial h(x,t)}{\partial t} + \alpha(h(x,t)) \geq 0\}&,
  \end{aligned}
\end{equation*}
then the controller $u(x,t) \in {K}_{cbf}(x,t)$ guarantees the safety of the system \eqref{eq:control_affine}.

\subsubsection{Gaussian Process Regression} \label{sec:gp}

Gaussian Process (GP) regression \cite{seeger2004gaussian} is a non-parametric Bayesian method widely employed for nonlinear regression tasks. Unlike parametric approaches, GP regression models the unknown target function as a probability distribution over functions, enabling predictions based on correlations among observed data points, rather than requiring parameter updates using large-scale datasets. This property makes GP regression particularly suitable for sparse and noisy datasets.

Mathematically, a GP can be fully characterized by its mean and kernel functions. For simplicity, we adopt a zero-mean Gaussian process expressed as
\begin{equation*}
f \sim \mathcal{GP}\left(0, k({x},{x}')\right),
\end{equation*}
where the kernel function $k({x},{x}'):\mathcal{X} \times \mathcal{X} \to \mathbb{R} $ quantifies the similarity or correlation between two arbitrary input points, and $\mathcal{X} \in \mathbb{R}^r$ represents the input space.

Consider a training dataset consisting of $N$ input-output pairs, denoted as $\{{x}_i,y_i\}_{i=1}^{N}$, where each input $x_i \in \mathcal{X}$ is associated with a label $y_i \in \mathbb{R}$.
The input and label datasets are defined as ${X} = [{x}_1, \ldots, {x}_N]^\mathsf{T}$ and ${Y} = [y_1, \ldots, y_N]^\mathsf{T}$, respectively.
For an arbitrary query point ${x}^* \in \mathcal{X}$, the predicted output ${y}^*$ follows a Gaussian distribution given by
\begin{equation} \label{eq:1a}
    {y}^* \sim \mathcal{N}(\mu({x}^*), \sigma^2({y}^*)),
\end{equation}
where the predictive mean and variance are computed by
\begin{subequations}
  \begin{align}
    \mu({x}^*) &= {k}^\mathsf{T}({x}^*)K^{-1}{Y}, \label{eq:1b}\\
    \sigma^2({y}^*) &= k({x}^*,{x}^*) - {k}^\mathsf{T}({x}^*)K^{-1}{k}({x}^*), \label{eq:1c}
  \end{align}
\end{subequations}
where $K \in \mathbb{R}^{N \times N}$ is the covariance matrix with elements $[K]_{ij}=k({x}_i,{x}_j),\, i,j \in \{1,\ldots,N\}$, 
and ${k}({x}^*)=[k({x}^*,{x}_1),\ldots,k({x}^*,{x}_N)]^\mathsf{T} \in \mathbb{R}^{N}$ represents the covariance vector between the query point ${x}^*$ and the input dataset.

\subsection{Problem Description}
We consider a wheeled differential-drive robot with state $x = [p_x,p_y,\theta]^\mathsf{T} \in \mathcal{X} \subseteq \mathbb{R}^2 \times [-\pi, \pi ) $ and control input $u = [v,\omega]^\mathsf{T} \in \mathbb{R}^2$. 
The robot dynamics is described by
\begin{equation}\label{eq:dynamic} 
  \begin{aligned}
    \begin{bmatrix}
      \dot{p}_x \\
      \dot{p}_y \\
      \dot{\theta}
    \end{bmatrix}
    &=
    \begin{bmatrix}
      \cos(\theta) & 0 \\
      \sin(\theta) & 0 \\
      0 & 1
    \end{bmatrix}
    \begin{bmatrix}
      v \\
      \omega
    \end{bmatrix}
  \end{aligned}.
\end{equation}
Let $\phi: \mathcal{X}  \to \mathbb{R}^2$ be the map from the robot state $x$ to its position $\phi(x) = [p_x,p_y]^\mathsf{T}$.

\textbf{Problem Statement:}
Consider a mobile robot with dynamics \eqref{eq:dynamic} navigating in an unknown dynamic environment toward a predefined goal. The robot is equipped with a LiDAR sensor that provides measurements of obstacles. 
The objective is to:
\begin{enumerate}
  \item Synthesize a dynamic control barrier function $h(x,t)$ online using LiDAR data, ensuring real-time adaptation to environmental changes.
    \item Design a controller $u(t)$ that guarantees safe navigation toward the goal while avoiding obstacles.
  \end{enumerate}

\section{Methodology} \label{sec:methodology}

\subsection{Local Perception}

The local perception module estimates obstacle positions and velocities in real-time using LiDAR sensor data, which are essential for synthesizing our proposed DLGP-CBF. 
Specifically, we construct a local obstacle grid map \(\mathcal{M}_{o,t} \in \{0,1\}^{W \times H}\), with width \(W\) and height \(H\), centered on the robot's current position in the global coordinate frame.
Each grid cell of the obstacle grid map \(\mathcal{M}_{o,t}(a,b)\) encodes the occupancy status of the corresponding spatial location and is updated at each time step $t$ as 
\begin{equation*} \small
  \mathcal{M}_{o,t}(a,b) = 
  \begin{cases} 1, & \text{if a LiDAR ray endpoint at time step } t  \\
    &\text{ falls within cell } (a,b),
    \\ 0, 
    & \text{otherwise}. \end{cases} \end{equation*}

Since our DLGP-CBF explicitly incorporates obstacle velocity information, a local velocity grid map \(\mathcal{M}_{v,t} \in \mathbb{R}^{W \times H \times 2}\) is also constructed. 
Each grid cell \(\mathcal{M}_{v,t}(a,b) \in \mathbb{R}^2\) stores the predicted velocity vector of the obstacle at the corresponding spatial location.
To obtain velocity estimates, occupied grid cells in the obstacle grid map \(\mathcal{M}_{o,t}\) are transformed into spatial points in the global frame, forming an array \(\mathcal{D}_t = [d_{1,x,t}, d_{1,y,t}, \ldots, d_{N,x,t}, d_{N,y,t}]^\mathsf{T} \in \mathbb{R}^{2N}\), where \(N\) denotes the number of occupied cells.

We use the DBSCAN algorithm \cite{ester1996density} to segment the points $\mathcal{D}_t$ into $N_{obs}$ clusters, denoted as $\mathcal{D}_t=\{\mathcal{D}_{1,t},\mathcal{D}_{2,t},\ldots,\mathcal{D}_{N_{obs},t}\}$. 
Following the clustering process, the minimum bounding ellipse (MBE) algorithm \cite{welzl2005smallest} is used to fit each cluster with a MBE. 
The set of MBEs is denoted as $\mathcal{E}_t = \{\mathcal{E}_{1,t},\mathcal{E}_{2,t},\ldots,\mathcal{E}_{N_{obs},t}\}$, where each ellipse $\mathcal{E}_{i,t} = [c_{i,x,t},c_{i,y,t},a_{i,t},b_{i,t},\theta_{i,t}]^\mathsf{T}$ is parameterized by its center coordinates $[c_{i,x,t}, c_{i,y,t}]^\mathsf{T}$, semi-major axis $a_{i,t}$, semi-minor axis $b_{i,t}$, and orientation $\theta_{i,t}$.

To establish temporal correspondences between ellipses at consecutive time steps, we compare the MBEs from the previous frame with those at the current frame. 
For this purpose, we define an affinity matrix $\mathcal{A} \in \mathbb{R}^{N_{obs,t-1} \times N_{obs,t}}$, where each element $\mathcal{A}_{i,j}$ represents the distance between the $i$-th ellipse at $t-1$ and the $j$-th ellipse at $t$. 
The Kuhn-Munkres algorithm \cite{kuhn1955hungarian} is applied to compute the optimal association between ellipses across frames. Matches with a distance exceeding a predefined threshold $d_{max}$ are classified as new obstacles, and newly detected ellipses are assigned fresh labels.
For MBEs successfully matched across time steps, we employ Kalman filtering \cite{welch1995introduction} to estimate their velocities. 
The state variable of each tracked MBE is defined as $\hat{\mathcal{E}}_{l} = [c_{l,x},c_{l,y},\dot{c}_{l,x},\dot{c}_{l,y},\ddot{c}_{l,x},\ddot{c}_{l,y},a_{l},b_{l},\theta_{l}]^\mathsf{T}$, where $\dot{c}_{l,x}$ and $\dot{c}_{l,y}$ denote the velocity, and $\ddot{c}_{l,x}$ and $\ddot{c}_{l,y}$ denote the acceleration of the ellipse.
Through iterative prediction and measurement updates, we obtain the estimated velocity $v_{l}=[\dot{c}_{l,x},\dot{c}_{l,y}]^\mathsf{T}$ of each ellipse $\hat{\mathcal{E}}_{l}$. 
The computed velocity $v_{l}$ is then assigned to the corresponding point cluster and subsequently mapped to the appropriate grid cell in the velocity grid map $\mathcal{M}_{v,t}$, 
thereby explicitly encoding obstacle velocity information for subsequent DLGP-CBF synthesis.

\subsection{Dynamic Log-GP CBF Synthesis}

We employ GP regression to construct the Dynamic Log Gaussian Process Control Barrier Function (DLGP-CBF), which quantifies the robot's safety level based on its state and obstacle information from LiDAR measurements. 
Specifically, the DLGP-CBF, denoted as $h(x,\mathcal{D})$, is defined as
\begin{equation} \label{eq:dlgp_cbf}
  h(x,\mathcal{D})= -c_{s}\log(\mu(\phi (x),\mathcal{D}))-d_{\text{shift}},
\end{equation}
where $\mu(\phi (x),\mathcal{D})$ denotes the predictive mean from the GP regression evaluated at the query position $\phi (x)$ with respect to the training dataset $\mathcal{D}$, 
and the user-defined parameters \(c_{s} \in \mathbb{R}^+\) and \(d_{\text{shift}} \in \mathbb{R}^+\) represent the scaling factor and the safety margin adjustment, respectively.

To form the training dataset for GP regression, positional information of the detected obstacles is extracted from the local obstacle grid map $\mathcal{M}_o$.
Specifically, occupied grid cells in $\mathcal{M}_o$ are transformed into spatial points within the global coordinate frame, forming the input dataset as
\begin{equation} \label{eq:input_dataset}
  \mathcal{D} = \left[d_{1}^\mathsf{T}, \ldots, d_{N}^\mathsf{T}\right]^\mathsf{T} \in \mathbb{R}^{2N},
\end{equation}
where $d_{i}=\left[d_{i,x}, d_{i,y}\right]^\mathsf{T}$ denotes the global coordinates of the $i$-th occupied grid cell, and $N$ represents the total number of occupied cells.
Similarly, the velocity information for obstacles is extracted from the velocity grid map $\mathcal{M}_v$ into the velocity array:
\begin{equation} \label{eq:velocity_dataset}
  \mathcal{V} = \left[v_{1}^\mathsf{T}, \ldots, v_{N}^\mathsf{T}\right]^\mathsf{T} \in \mathbb{R}^{2N},
\end{equation}
where $v_{i}=\left[v_{i,x}, v_{i,y}\right]^\mathsf{T}$ denotes the velocity vector of the $i$-th occupied grid cell.
Each input data point in $\mathcal{D}$ is assigned the label $y_i=1$, forming the corresponding label dataset $Y = [y_1, \ldots, y_N]^\mathsf{T} \in \mathbb{R}^N$.

We use the zero-mean GP regression model with the squared exponential (SE) kernel function $k \left(p,p'\right)= {  e^{ -\frac{||p-p'||^2}{2l^2}}} $ to construct the DLGP-CBF, where $l$ denote the length scale of the SE kernel.
The predictive mean $\mu(\phi (x),\mathcal{D})$ is computed by
\begin{equation}
  \mu(\phi (x),\mathcal{D}) = {\tilde{k}}^\mathsf{T} K^{-1}{Y},
\end{equation}
where ${\tilde{k}}=[k(\phi (x),{d}_1),\ldots,k(\phi (x),{d}_N)]^\mathsf{T} \in \mathbb{R}^{N}$ denotes the covariance vector between the query point $\phi (x)$ and the input dataset $\mathcal{D}$,
and $K \in \mathbb{R}^{N \times N}$ is the covariance matrix with elements $[K]_{ij}=k({d}_i,{d}_j),\, i,j \in \{1,\ldots,N\}$.
The SE kernel ensures smoothness and continuity in the resulting DLGP-CBF, and the covariance matrix \(K\) is positive definite due to the properties of the kernel.

The following theorem characterizes the properties of the DLGP-CBF $h(x,\mathcal{D})$.
\begin{thm} 
  Consider the input dataset $\mathcal{D}$ defined in \eqref{eq:input_dataset} consisting of obstacle positions, and the corresponding label dataset $Y=\mathbf{1}_N$, where $\mathbf{1}_N$ is the vector of ones.
  The DLGP-CBF $h(x,\mathcal{D})$ defined in \eqref{eq:dlgp_cbf} maps the robot's state to the range $[-d_{\text{shift}},+\infty )$ based on its distance from obstacles.
  Let $d_{\text{min}} = \min_{d \in \mathcal{D}} ||\phi(x)-d||$ denote the minimum distance between the robot and the obstacle positions within the input dataset $\mathcal{D}$.
  Then, the DLGP-CBF $h(x,\mathcal{D})$ satisfies: 
  \begin{itemize} 
    \item $h(x,\mathcal{D}) = -d_{\text{shift}}$, when $d_{\text{min}}=0$;
    \item $h(x,\mathcal{D}) \to +\infty$, as $d_{\text{min}} \to +\infty$;
    \item $h(x,\mathcal{D})$ is a strictly decreasing function of $d_{\text{min}}$.
    \end{itemize} 
\end{thm}

\begin{proof}
  When $d_{\text{min}}=0$, i.e., $\phi(x)=d_j$ for some $j \in \{1,\ldots,N\}$, the predictive mean $\mu(\phi (x),\mathcal{D})$ can be expressed as
  \begin{equation*} \small
    \mu(\phi (x),\mathcal{D}) = [K]_{:,j}^\mathsf{T}K^{-1}{\mathbf{1}_N } = e_j^\mathsf{T}{\mathbf{1}_N } = 1,
  \end{equation*}
  where $[K]_{:,j}$ denotes the $j$-th column of matrix $K$ and $e_j$ denotes the $j$-th standard basis vector.
  Thus, \(h(x,\mathcal{D}) = -c_s \log(1) - d_{\text{shift}} = -d_{\text{shift}}\).
  As \(d_{\text{min}} \to +\infty\), the covariance vector \(\tilde{k}(\phi(x))\) approaches to the zero vector, i.e., \(\lim_{d_{\text{min}} \to +\infty} \tilde{k}(\phi(x)) = \mathbf{0}_N\), leading to \(\mu(\phi(x),\mathcal{D}) \to 0\), and hence \(h(x,\mathcal{D}) \to +\infty\).

  Additionally, the SE kernel satisfies $\frac{\partial k}{\partial ||p-p'||} < 0$, which implies that as \(d_{\text{min}}\) increases, each entry in \(\tilde{k}(\phi(x))\) decreases or remains the same.
  Since the predictive mean \(\mu(\phi(x),\mathcal{D})\) is a positive linear combination of the entries in \(\tilde{k}(\phi(x))\), it is strictly decreasing with respect to \(d_{\text{min}}\).
  After applying the logarithmic transformation, the DLGP-CBF is strictly increasing with respect to \(d_{\text{min}}\), proving the stated properties.
  \end{proof}

\begin{remark}
  The parameter $d_{\text{shift}}$ in the DLGP-CBF formulation \eqref{eq:dlgp_cbf} is specially designed to adjust the size of the safe set around obstacles.
  A large $d_{\text{shift}}$ value results in a larger zero-level set, thereby inducing a more conservative obstacle avoidance behavior. 
  Conversely, choosing a smaller $d_{\text{shift}}$ value may allow the robot to approach obstacles more closely, leading to a more aggressive navigation strategy.
  \end{remark}

  \begin{remark} 
    The proposed DLGP-CBF, defined in \eqref{eq:dlgp_cbf}, introduces two distinct innovations compared to conventional GP-based methods. 
    First, through the logarithmic transformation, the DLGP-CBF ensures that the barrier function values and gradients remain informative and non-vanishing at locations distant from the training data (see Fig.~\ref{cbf}), thus significantly enhancing the continuity and reliability of barrier information. 
    Second, the DLGP-CBF explicitly depends on obstacle positional data encoded within the training dataset, enabling the model to dynamically reflect obstacle positions and motions. These properties collectively improve the robot's capability to perform accurate, timely, and efficient obstacle avoidance behavior, consequently enhancing safety in dynamic environments.

  \end{remark}

\begin{figure}[!t]
  \centering
  \includegraphics[width=3.4in]{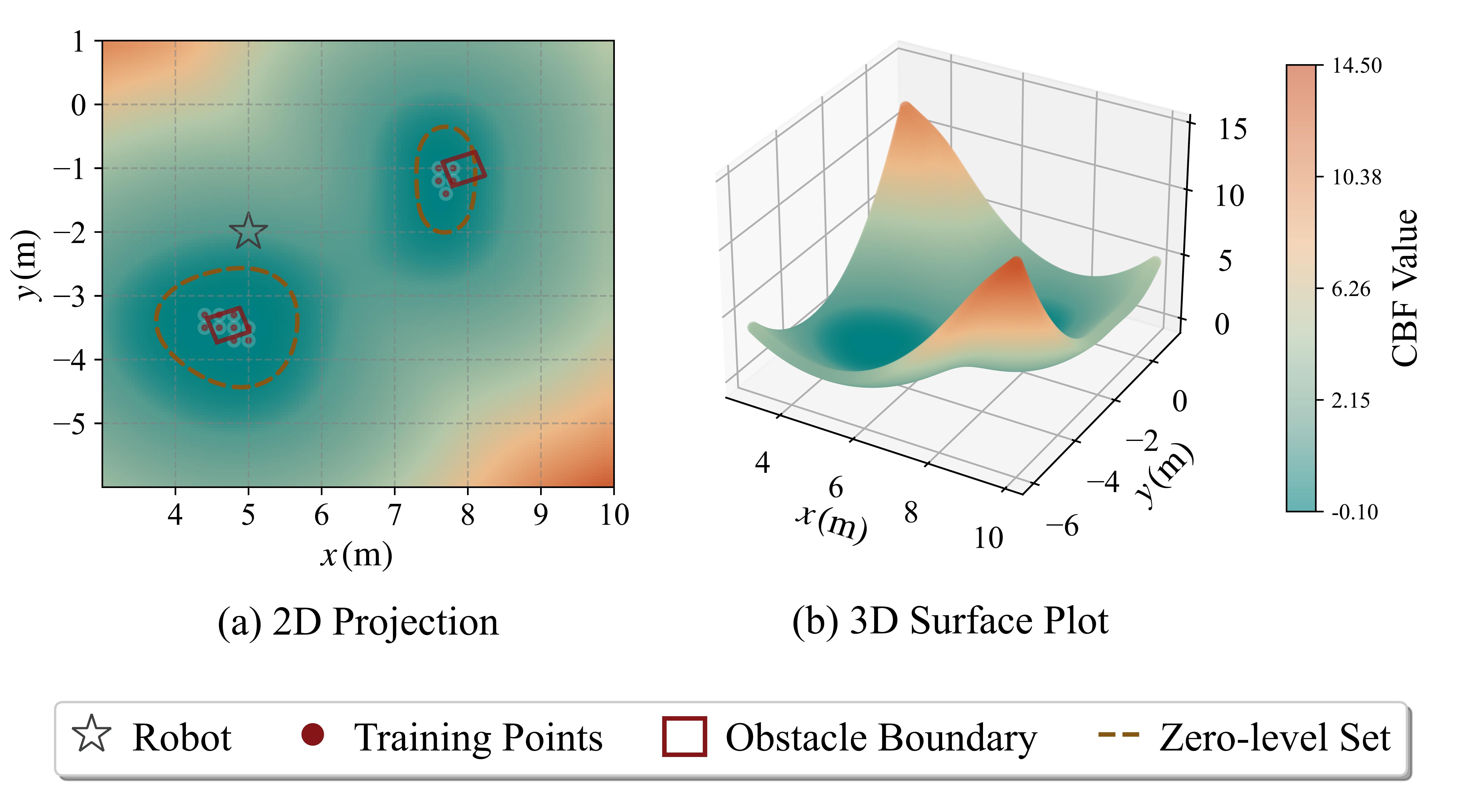}
  \caption{
    Visualization of the DLGP-CBF function with $c_s = 1$, $d_{\text{shift}} = 0.1$, and the SE kernel length scale $l = 0.9$.
    }
    \vspace{-0.5cm}
  \label{cbf}
  \end{figure}

\subsection{Safe Control via DLGP-CBF}

At each control step $t$, the latest updated obstacle grid map $\mathcal{M}_{o,t}$ and the velocity grid map $\mathcal{M}_{v,t}$ are utilized to construct the input dataset $\mathcal{D}_t$ and the velocity dataset $\mathcal{V}_t$ as defined in \eqref{eq:input_dataset} and \eqref{eq:velocity_dataset}, respectively.
The DLGP-CBF is then updated using these newly constructed datasets. 
For notational simplicity, we omit the time dependence of \( \mathcal{D}_t \) and \( \mathcal{V}_t \) in the subsequent discussion.

Given the dynamics \eqref{eq:dynamic}, we derive the safe control input by solving 
the following quadratic programming (QP) problem:
\begin{equation} \label{eq:qp}
  \begin{aligned}
    \min_{u \in \mathbb{R}^2} & \quad   ||u-u_{\text{nom}}(t)||^2, \\
    \text{s.t.} & \quad \mathcal{L}_fh(x,\mathcal{D}) + \mathcal{L}_g h(x,\mathcal{D})u \\
    & \quad  + \frac{\partial h(x,\mathcal{D})}{\partial t} + \alpha(h(x,\mathcal{D})) \geq 0, \\
  \end{aligned}
\end{equation}
where $u_{\text{nom}}(t)$ is the nominal control input that may not inherently guarantee safety.

To construct the CBF constraint explicitly, we need to compute both the spatial gradient $\frac{\partial h(x,\mathcal{D})}{\partial x}$, reflecting the influence of the robot's state on the barrier function, and the time derivative $\frac{\partial h(x,\mathcal{D})}{\partial t}$, capturing the effect of moving obstacles on the barrier function.
The partial derivative of $h(x,\mathcal{D})$  with respect to the robot state $x$ is calculated as
\begin{equation} \small
  \begin{aligned}
  & \frac{\partial h(x,\mathcal{D})}{\partial x} 
  =   -c_{s}\frac{\left[
    Y^\mathsf{T} K^{-1} \frac{\partial \tilde{k}}{\partial \phi(x)} ,
   0
   \right]^\mathsf{T}}{\mu(\phi(x),\mathcal{D})} ,
\end{aligned}
  \end{equation}
where $\frac{\partial \tilde{k}}{\partial \phi(x) }\in \mathbb{R}^{N \times 2}$ is the derivative of the covariance vector with respect to the robot's position $\phi(x)$.
The derivative $\frac{\partial \tilde{k}}{\partial \phi(x) }$ is given by
  \begin{equation}
     \frac{\partial \tilde{k}}{\partial \phi(x)}  =
    \begin{bmatrix}
      -\frac{1}{l^2} k_{\text{se}}(\phi(x), d_1) (\phi(x) - d_1)^\mathsf{T} \\
    \vdots \\
    -\frac{1}{l^2} k_{\text{se}}(\phi(x), d_N) (\phi(x) - d_N)^\mathsf{T}
    \end{bmatrix},
    \end{equation}
where $d_i$ is the global coordinates of the $i$-th training point in the input dataset $\mathcal{D}$.

Additionally, as $h(x,\mathcal{D})$ explicitly depends on obstacle positions encoded within the input dataset $\mathcal{D}$, 
changes in obstacle positions directly influence the barrier function through the time derivative, calculated by
\begin{equation} \small
  \begin{aligned}
   \frac{\partial h(x,\mathcal{D})}{\partial t} 
  =&   \sum_{r=1}^{N}   \frac{\partial h(x,d_1,\ldots,d_N)}{\partial d_r} \frac{\partial d_r}{\partial t}\\
  =& \frac{c_{s}Y^\mathsf{T} K^{-1} \otimes \tilde{k}^\mathsf{T} K^{-1} }{\mu(\phi(x),\mathcal{D})} \sum_{r=1}^{N}  \frac{\partial K}{\partial d_r} v_r
    \\
  &- \frac {c_{s}Y^\mathsf{T} K^{-1}}{\mu(\phi(x),\mathcal{D})} \sum_{r=1}^{N}  \frac{\partial \tilde{k} }{\partial d_r} v_r ,
\end{aligned}
  \end{equation}
where $\otimes$ denotes the Kronecker product, $v_r$ denotes the velocity of the $r$-th training point $d_r$,
and the term $\frac{\partial K}{\partial d_r} \in \mathbb{R}^{N^2 \times 2}$ and $\frac{\partial \tilde{k} }{\partial d_r} \in \mathbb{R}^{N \times 2}$ denote the partial derivatives of the covariance matrix $K$ and the covariance vector $\tilde{k}$ with respect to the $r$-th training point $d_r$, respectively.
Specifically, the derivative $\frac{\partial K}{\partial d_r}$ is given by
    \begin{equation} \small
    \begin{aligned}
    &\frac{\partial K}{\partial d_r}  = \frac{\partial \mathrm{vec}(K)}{\partial d_r}
    =\left[ \frac{\partial K_{ij}}{\partial d_r} \right] \\
    =&\left[
      -\frac{1}{l^2}K_{ij}\left(\   \delta_{ri}(d_i-d_j)^\mathsf{T} + \delta_{rj}(d_j-d_i)^\mathsf{T} \right) 
    \right]_{i,j=1}^N,
    \end{aligned}
    \end{equation}
where $\mathrm{vec}(K)$ denotes the vectorization operation which stacks the columns of the matrix $K$ into a single column vector.
The derivative $\frac{\partial\tilde{k}}{\partial d_r} $ is expressed as
\begin{equation}
  \frac{\partial\tilde{k}}{\partial d_r}  =
  \begin{bmatrix}
    \frac{1}{l^2} k_{\text{se}}(\phi(x), d_1) (\phi(x) - d_1)^\mathsf{T} \delta_{r1}  \\
  \vdots \\
  \frac{1}{l^2} k_{\text{se}}(\phi(x), d_N) (\phi(x) - d_N)^\mathsf{T} \delta_{rN}
  \end{bmatrix}.
  \end{equation}

Finally, once the terms $\frac{\partial h(x,\mathcal{D})}{\partial x}$ and $\frac{\partial h(x,\mathcal{D})}{\partial t}$ are computed, the safe control input $u(t)$ can be obtained by solving the QP problem \eqref{eq:qp}, thereby ensuring the robot's safety in dynamic environments.

\section{Simulations} \label{sec:simulation}

\subsection{Experiment Setup}

To evaluate the performance of the proposed method, simulations were conducted using the TurtleBot3 robot in the Gazebo simulator. The robot operates in an unknown dynamic environment and is equipped with a Velodyne VLP-16 LiDAR sensor to perceive its surroundings.
The robot starts at the initial position \([-8, 3]^\mathsf{T}\) and is tasked with reaching a predefined goal position \(p_{\text{goal}} = [10, -2]^\mathsf{T}\) while avoiding both static and dynamic obstacles.
The nominal controller $u_{\text{nom}}(t)$ is designed as a go-to-goal controller that steers the robot towards a predefined goal position, and is defined as $u_{\text{nom}}(t)=u_{\text{max}}\frac{p_{\text{goal}}- \phi(x)}{\|p_{\text{goal}}- \phi(x)\|} $, where $u_{\text{max}} \in \mathbb{R}^+$ is the maximum control input magnitude.

Since the relative degree of the DLGP-CBF \eqref{eq:dlgp_cbf} with respect to the angular velocity $\omega$ is $2$, (i.e., the CBF must be differentiated twice to expose $\omega$ explicitly), a virtual leading point is introduced to simplify the control design. 
For more details, we refer the reader to \cite{cortes2017coordinated}.
Specifically, the virtual leading point is defined at a distance $l \in \mathbb{R}^+$ ahead of the robot along its heading direction as $p_{\text{lead}} = \phi(x) + l [\cos(\theta), \sin(\theta)]^\mathsf{T}$, and it follows a single-integrator dynamics $\dot{p}_{\text{lead}} =u_{\text{lead}}$.
The safe control input \(u_{\text{lead}}(t)\) is computed by solving the QP problem~\eqref{eq:qp}. The corresponding control input \([v(t), \omega(t)]^\mathsf{T}\) for the differential-drive robot, governed by~\eqref{eq:dynamic}, is then recovered via a near-identity diffeomorphism:
\begin{equation}
  \begin{bmatrix}
  v(t) \\
  w(t)
  \end{bmatrix}
  =
  \begin{bmatrix}
  \cos(\theta(t)) & \sin(\theta(t)) \\
  -\frac{1}{l} \sin(\theta(t)) & \frac{1}{l} \cos(\theta(t))
  \end{bmatrix}
  u_{\text{lead}}(t).
  \end{equation}

The hyperparameter for SE kernel is set as $l=0.9$.
For the DLGP-CBF defined in \eqref{eq:dlgp_cbf}, the user-defined parameters are chosen as $c_{s}=1$ and $d_{\text{shift}}=0.1$.
Additionally, the extended class $\mathcal{K}_\infty $ function used in the CBF constraint \eqref{eq:qp} is selected as $\alpha(h(x,\mathcal{D})) = 0.2h(x,\mathcal{D})$.

\subsection{Simulations}

We evaluate the performance of the proposed DLGP-CBF method in a dynamic environment and compare it with two baseline approaches: GP-CBF~\cite{keyumarsi2023lidar} and MPC-DCBF~\cite{jian2023dynamic}.

\subsubsection{Performance of DLGP-CBF}
\begin{figure}[!t]
  \centering
  \includegraphics[width=2.9in]{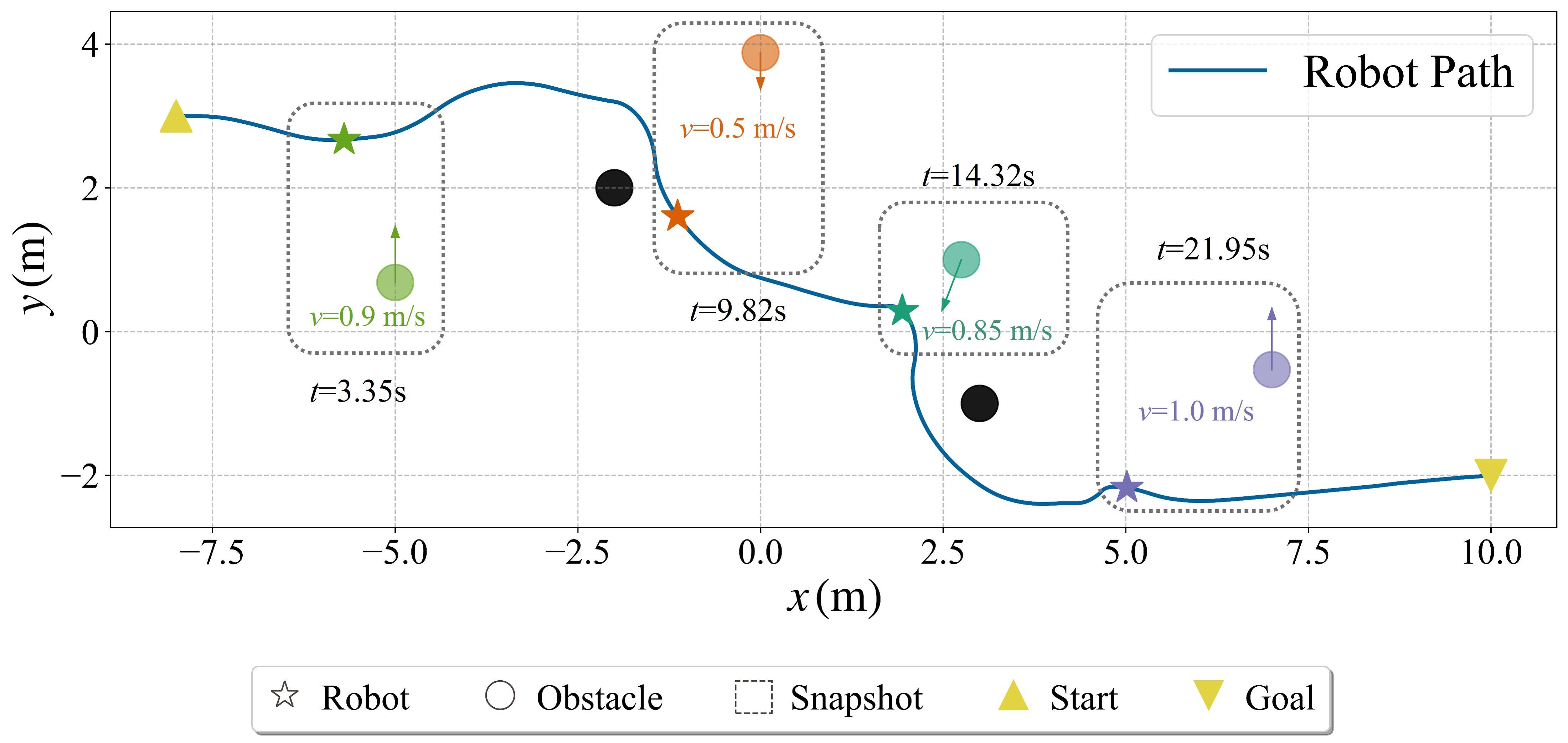}
  \caption{
    Trajectory generated by the DLGP-CBF method. 
    Dynamic and static obstacles are shown as colored and black circles, respectively. 
    Dotted rectangles indicate the positions of the robot and dynamic obstacles at selected timestamps.
    }
  \label{fig::trajactory}
  \end{figure}

  \begin{figure}[!t]
    \centering
    \includegraphics[width=2.9in]{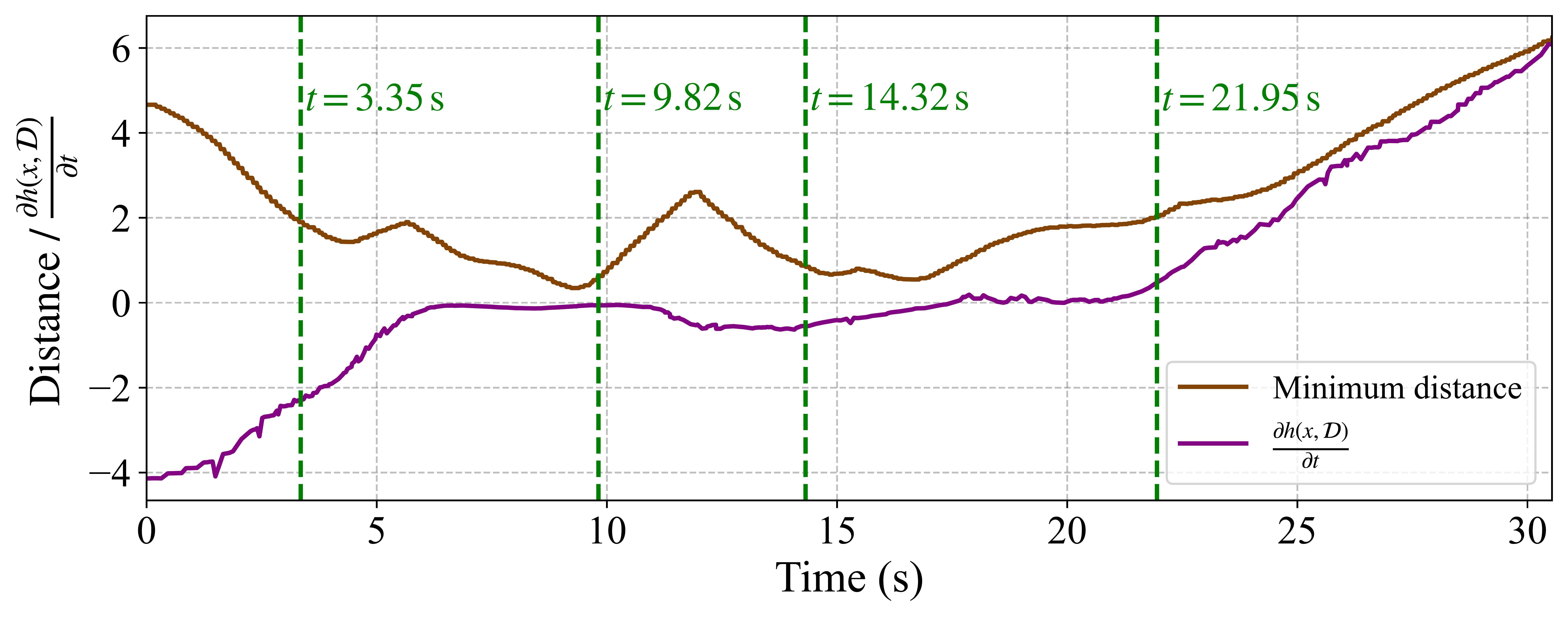}
    \caption{
      Evolution of the minimum distance to obstacles and the time derivative \(\frac{\partial h}{\partial t}\) over time. 
      Green vertical lines correspond to the snapshot timestamps in Fig.~\ref{fig::trajactory}.
        }
        \vspace{-0.5cm}
    \label{fig::grad}
    \end{figure}

Fig.~\ref{fig::trajactory} shows the trajectory generated by DLGP-CBF. The robot successfully reaches the goal while avoiding both static and dynamic obstacles. The snapshots at different timestamps illustrate the robot's timely and proactive responses to approaching dynamic obstacles.
To further analyze this behavior, Fig.~\ref{fig::grad} presents the evolution of the minimum distance to obstacles and the time derivative \(\frac{\partial h(x,\mathcal{D})}{\partial t}\). 
At critical moments corresponding to the snapshots in Fig.~\ref{fig::trajactory}, the time derivative becomes negative as obstacles move toward the robot, indicating a decreasing barrier value. 
This triggers early avoidance maneuvers, yielding smoother and safer trajectories.

It is worth noting that the update time of the DLGP-CBF depends on the training dataset size. 
However, GP regression inherently leverages correlations between data points, allowing it to perform effectively with sparse training datasets \cite{lyu}. 
Therefore, overly dense training data may introduce redundant information and increase computational complexity unnecessarily. In practice, the dataset can be downsampled to balance computational efficiency and control performance. 
In our simulations, the average training dataset size is 30, resulting in an average inference time of approximately 21 ms, satisfying real-time control requirements.

\subsubsection{Comparison with Baseline Methods}

We compare DLGP-CBF with GP-CBF and MPC-DCBF under the same experimental setup. The resulting trajectories are shown in Fig.~\ref{comparison}, and the quantitative results are summarized in Table \uppercase\expandafter{\romannumeral1}, using four metrics: minimum distance to obstacles during the whole navigating process, arrival time, and the variances of linear and angular velocities.
As observed, all methods achieve collision-free navigation. However, DLGP-CBF consistently maintains a larger minimum distance, achieves a shorter arrival time, and exhibits lower angular velocity variance, indicating enhanced safety, efficiency, and control smoothness.
While MPC-DCBF also accounts for obstacle dynamics, its reliance on per-obstacle geometric fitting increases complexity and may degrade control consistency.
In contrast, DLGP-CBF constructs a single unified barrier function directly from LiDAR data, enabling more stable and scalable control performance, as shown in Table \uppercase\expandafter{\romannumeral1}.

\begin{figure}[!t]
  \centering
  \includegraphics[width=2.9in]{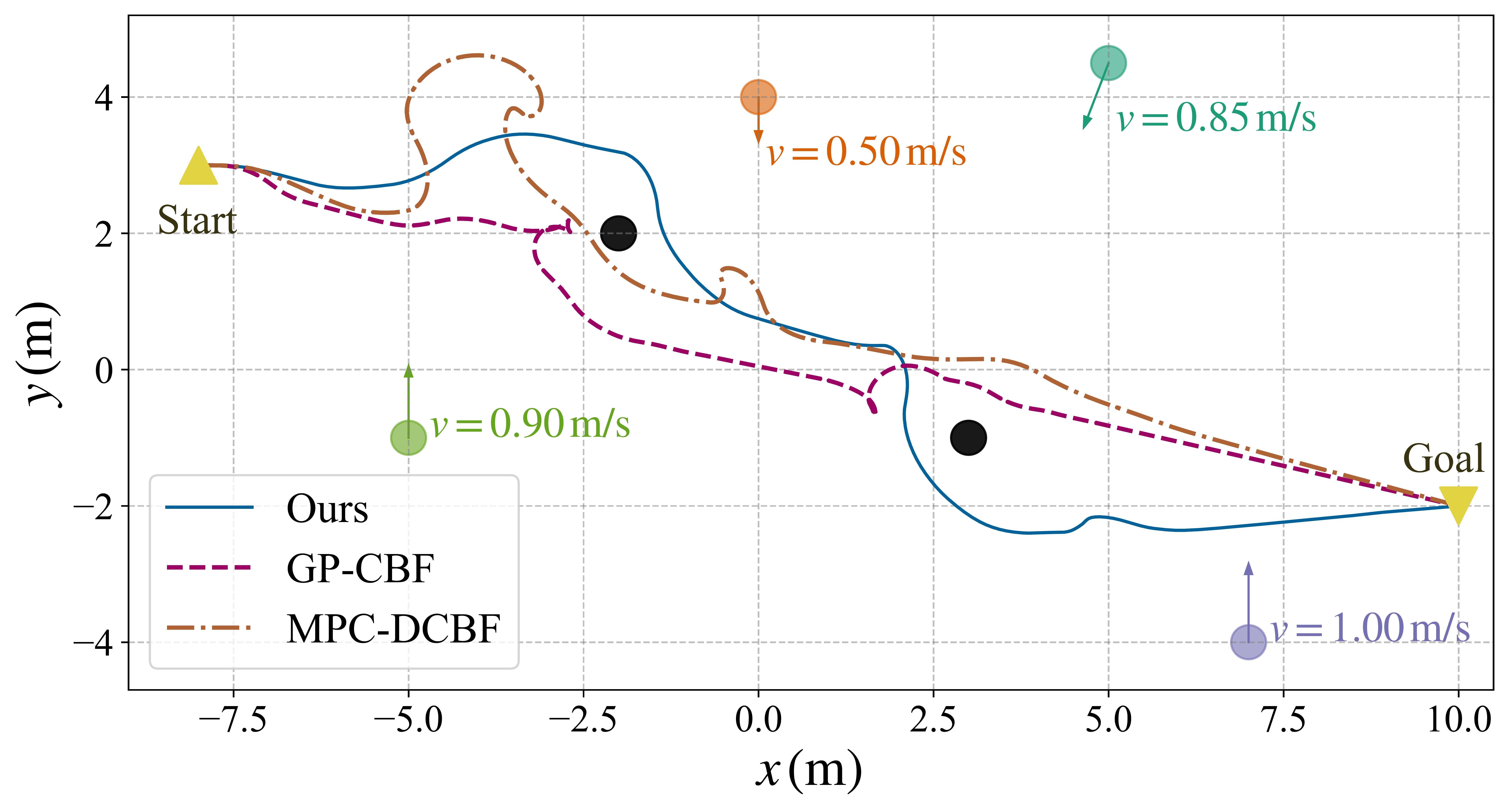}
  \caption{Comparison of trajectories generated by different methods.
  The colored circles represent dynamic obstacles with different velocities and the black circles represent static obstacles. 
  }
  \vspace{-0.5cm}
  \label{comparison}
  \end{figure}

\begin{table}[htb]   
    \centering \label{table:time}
    \caption{Performance comparison with baseline methods.}  
    \tabcolsep=3mm
      \begin{tabular}{>{\centering\arraybackslash}m{1.5cm} 
      >{\centering\arraybackslash}m{1cm}
      >{\centering\arraybackslash}m{1cm}
      >{\centering\arraybackslash}m{0.8cm}
      >{\centering\arraybackslash}m{0.8cm}}   
    \toprule
    \multirow{2}{*}{\textbf{Algorithms}} 
    & \multirow{2}{*}{\makecell{\textbf{Min.}\\\textbf{Dist. (m)}}} 
    & \multirow{2}{*}{\makecell{\textbf{Arrival}\\\textbf{Time (s)}}} 
    & \multicolumn{2}{c}{\textbf{Speed Variance}} \\ 
    \cmidrule(lr){4-5}
    & & & \textbf{Linear} & \textbf{Angular} \\  
            \midrule
    MPC-DCBF & 0.129 & 33.867 & \textbf{0.110} & 0.702 \\ 
    GP-CBF   & 0.161 & 33.766& 0.143 & 0.575 \\  
    \textbf{Ours}    & \textbf{0.346} & \textbf{30.533} & 0.112 & \textbf{0.339} \\
    \bottomrule
    \end{tabular}   
    \end{table}
    \vspace{-0.5cm}

\section{Conclusion}

In this paper, we have proposed the DLGP-CBF to address the safety-critical navigation problem for robots operating in unknown and dynamic environments. The DLGP-CBF explicitly incorporates dynamic obstacle velocity information, enabling proactive and timely responses to moving obstacles. Additionally, through a logarithmic Gaussian Process formulation, the proposed method provides more accurate and informative barrier function values and gradients, even in regions distant from observed obstacles. 
Simulation results have validated the effectiveness of our method.

\bibliographystyle{IEEEtran} 

\bibliography{refs}

@inproceedings{ames2019control,
  title        = {Control barrier functions: Theory and applications},
  author       = {Ames, Aaron D and Coogan, Samuel and Egerstedt, Magnus and Notomista, Gennaro and Sreenath, Koushil and Tabuada, Paulo},
  booktitle    = {2019 18th European Control Conference (ECC)},
  pages        = {3420--3431},
  year         = {2019},
  organization = {IEEE}
}

@inproceedings{jian2023dynamic,
  title        = {Dynamic control barrier function-based model predictive control to safety-critical obstacle-avoidance of mobile robot},
  author       = {Jian, Zhuozhu and Yan, Zihong and Lei, Xuanang and Lu, Zihong and Lan, Bin and Wang, Xueqian and Liang, Bin},
  booktitle    = {2023 IEEE International Conference on Robotics and Automation (ICRA)},
  pages        = {3679--3685},
  year         = {2023},
  organization = {IEEE}
}

@article{dawson2022learning,
  title={Learning safe, generalizable perception-based hybrid control with certificates},
  author={Dawson, Charles and Lowenkamp, Bethany and Goff, Dylan and Fan, Chuchu},
  journal={IEEE Robotics and Automation Letters},
  volume={7},
  number={2},
  pages={1904--1911},
  year={2022},
  publisher={IEEE}
}

@inproceedings{autenrieb2023safe,
  title={Safe and stable adaptive control for a class of dynamic systems},
  author={Autenrieb, Johannes and Annaswamy, Anuradha},
  booktitle={2023 62nd IEEE Conference on Decision and Control (CDC)},
  pages={5059--5066},
  year={2023},
  organization={IEEE}
}

@INPROCEEDINGS{10160805,
  author={Abdi, Hossein and Raja, Golnaz and Ghabcheloo, Reza},
  booktitle={2023 IEEE International Conference on Robotics and Automation (ICRA)}, 
  title={Safe Control using Vision-based Control Barrier Function (V-CBF)}, 
  year={2023},
  volume={},
  number={},
  pages={782-788},
  doi={10.1109/ICRA48891.2023.10160805}}

@inproceedings{lie2023formation,
  title={Formation Control of Underactuated AUVs Using the Hand Position Concept},
  author={Lie, Erling S and Matou{\v{s}}, Josef and Pettersen, Kristin Y},
  booktitle={2023 62nd IEEE Conference on Decision and Control (CDC)},
  pages={1412--1419},
  year={2023},
  organization={IEEE}
}

@article{wang2017safety,
  title     = {Safety barrier certificates for collisions-free multirobot systems},
  author    = {Wang, Li and Ames, Aaron D and Egerstedt, Magnus},
  journal   = {IEEE Transactions on Robotics},
  volume    = {33},
  number    = {3},
  pages     = {661--674},
  year      = {2017},
  publisher = {IEEE}
}

@article{zhang2025robust,
  title={Robust Dual-Filter Safety Control for Mobile Robots in Dynamic Multiobstacle Environments},
  author={Zhang, Yu and Kong, Linghuan and Yu, Xinbo and He, Wei and Knoll, Alois},
  journal={IEEE/ASME Transactions on Mechatronics},
  year={2025},
  publisher={IEEE}
}

@ARTICLE{9888130,
  author={Khan, Mouhyemen A. and Ibuki, Tatsuya and Chatterjee, Abhijit},
  journal={IEEE Access}, 
  title={Gaussian Control Barrier Functions: Non-Parametric Paradigm to Safety}, 
  year={2022},
  volume={10},
  number={},
  pages={99823-99836},
  doi={10.1109/ACCESS.2022.3206372}}

@inproceedings{srinivasan2020synthesis,
  title={Synthesis of control barrier functions using a supervised machine learning approach},
  author={Srinivasan, Mohit and Dabholkar, Amogh and Coogan, Samuel and Vela, Patricio A},
  booktitle={2020 IEEE/RSJ International Conference on Intelligent Robots and Systems (IROS)},
  pages={7139--7145},
  year={2020},
  organization={IEEE}
}

@INPROCEEDINGS{9981177,
  author={Lafmejani, Amir Salimi and Berman, Spring and Fainekos, Georgios},
  booktitle={2022 IEEE/RSJ International Conference on Intelligent Robots and Systems (IROS)}, 
  title={NMPC-LBF: Nonlinear MPC with Learned Barrier Function for Decentralized Safe Navigation of Multiple Robots in Unknown Environments}, 
  year={2022},
  volume={},
  number={},
  pages={10297-10303},
  doi={10.1109/IROS47612.2022.9981177}}

@inproceedings{zhang2024online,
  title={Online efficient safety-critical control for mobile robots in unknown dynamic multi-obstacle environments},
  author={Zhang, Yu and Tian, Guangyao and Wen, Long and Yao, Xiangtong and Zhang, Liding and Bing, Zhenshan and He, Wei and Knoll, Alois},
  booktitle={2024 IEEE/RSJ International Conference on Intelligent Robots and Systems (IROS)},
  pages={12370--12377},
  year={2024},
  organization={IEEE}
}

@article{keyumarsi2023lidar,
  title={LiDAR-based online control barrier function synthesis for safe navigation in unknown environments},
  author={Keyumarsi, Shaghayegh and Atman, Made Widhi Surya and Gusrialdi, Azwirman},
  journal={IEEE Robotics and Automation Letters},
  volume={9},
  number={2},
  pages={1043--1050},
  year={2023},
  publisher={IEEE}
}

@inproceedings{ester1996density,
author = {Ester, Martin and Kriegel, Hans-Peter and Sander, J\"{o}rg and Xu, Xiaowei},
title = {A density-based algorithm for discovering clusters in large spatial databases with noise},
year = {1996},
publisher = {AAAI Press},
booktitle = {Proceedings of the Second International Conference on Knowledge Discovery and Data Mining},
pages = {226–231},
numpages = {6},
location = {Portland, Oregon},
}

@inproceedings{welzl2005smallest,
  title={Smallest enclosing disks (balls and ellipsoids)},
  author={Welzl, Emo},
  booktitle={New Results and New Trends in Computer Science: Graz, Austria, June 20--21, 1991 Proceedings},
  pages={359--370},
  year={2005},
  organization={Springer}
}

@article{kuhn1955hungarian,
  title={The Hungarian method for the assignment problem},
  author={Kuhn, Harold W},
  journal={Naval research logistics quarterly},
  volume={2},
  number={1-2},
  pages={83--97},
  year={1955},
  publisher={Wiley Online Library}
}

@article{welch1995introduction,
author = {Welch, Greg and Bishop, Gary},
year = {2006},
month = {01},
pages = {},
title = {An Introduction to the Kalman Filter},
volume = {8},
journal = {Proc. Siggraph Course}
}

@article{seeger2004gaussian,
  title={Gaussian processes for machine learning},
  author={Seeger, Matthias},
  journal={International Journal of Neural Systems},
  volume={14},
  number={02},
  pages={69--106},
  year={2004},
  publisher={World Scientific}
}

@article{cortes2017coordinated,
  title={Coordinated control of multi-robot systems: A survey},
  author={Cort{\'e}s, Jorge and Egerstedt, Magnus},
  journal={SICE Journal of Control, Measurement, and System Integration},
  volume={10},
  number={6},
  pages={495--503},
  year={2017},
  publisher={Taylor \& Francis}
}

@InProceedings{lyu,
author="Lyu, Chenyi
and Liu, Xingchi
and Mihaylova, Lyudmila",
editor="Panoutsos, George
and Mahfouf, Mahdi
and Mihaylova, Lyudmila S.",
title="Review of Recent Advances in Gaussian Process Regression Methods",
booktitle="Advances in Computational Intelligence Systems",
year="2024",
publisher="Springer Nature Switzerland",
address="Cham",
pages="226--237",
isbn="978-3-031-55568-8"
}

\end{document}